\documentclass[letterpaper, 10 pt, conference]{ieeeconf}  

\IEEEoverridecommandlockouts                              

\usepackage[T1]{fontenc}
\usepackage{graphicx} 
\usepackage{amsmath} 
\usepackage{amssymb}  
\usepackage{mathtools}
\usepackage[hyphens]{url}
\usepackage{xcolor}

\usepackage[ruled,vlined,linesnumbered]{algorithm2e}
\usepackage{algorithmic,algorithm2e,float}

\SetKw{Continue}{continue}

\usepackage{amsthm}
\theoremstyle{definition}
\newtheorem{problem}{Problem}
\newtheorem{thm}{Theorem}
\newtheorem{assumption}{Assumption}
\newtheorem{definition}{Definition}

\newtheorem{exmp}{Example}
\newtheorem*{rem}{Remark}

\newtheorem{lem}{Lemma}

\newcommand\oprocendsymbol{\hbox{$\bullet$}}
\newcommand\oprocend{\relax\ifmmode\else\unskip\hfill\fi\oprocendsymbol}

\newcommand{\dc}{DiskCover}
\newcommand{\hc}{HexCover}
\newcommand{\dct}{DiskCoverTour}
\newcommand{\hct}{HexCoverTour}

\title{\LARGE \bf Approximation Algorithms for Robot Tours in Random Fields with Guaranteed Estimation Accuracy}

\author{Shamak Dutta, Nils Wilde, Pratap Tokekar, and Stephen L. Smith
\thanks{This research is supported in part by the Natural Sciences and Engineering
Research Council of Canada (NSERC) and by Nutrien Ltd.}
\thanks{S.\ Dutta and S.\ L.\ Smith are with the Department of Electrical and Computer Engineering,
        University of Waterloo, Canada 
        \{{\tt\small stephen.smith, shamak.dutta\}@uwaterloo.ca}.  N.\ Wilde is with the Cognitive Robotics Department, Delft University of Technology, Netherlands ({\tt\small N.Wilde@tudelft.nl}). P.\ Tokekar is with the Department of Computer Science, University of Maryland ({\tt\small tokekar@umd.edu}).}%
}
\begin{document}
\maketitle
\thispagestyle{empty}
\pagestyle{empty}

\begin{abstract}
    We study the sample placement and shortest tour problem for robots tasked with mapping environmental phenomena modeled as stationary random fields. The objective is to minimize the resources used (samples or tour length) while guaranteeing estimation accuracy. We give approximation algorithms for both problems in convex environments. These improve previously known results, both in terms of theoretical guarantees and in simulations. In addition, we disprove an existing claim in the literature on a lower bound for a solution to the sample placement problem.
\end{abstract}

\section{Introduction} \label{section:introduction}
In environmental monitoring, robots are tasked with mapping physical phenomena with desired accuracy while minimizing the resources used. For example, farmers want accurate nutrient maps of their fields to direct fertilizer usage which will subsequently maximize their crop yield. Unfortunately, soil sampling is time-consuming and expensive \cite{webster2007geostatistics}. It is advantageous to minimize the number of samples taken to build a nutrient map. In post-disaster environments, there is a need to map the spatial distribution of radioactivity using unmanned aerial systems in a short period of time \cite{connor2020radiological}. This involves a combination of deciding the measurement locations as well as planning a short path through the environment with the goal of building an accurate map.

We consider the setting where the phenomena is modeled as a stationary random field over a compact two-dimensional environment. The covariance structure of the stationary field is given by the commonly used squared-exponential function. This probabilistic model is popular in spatial statistics \cite{webster2007geostatistics,cressie2015statistics} since it allows for predictions at unobserved locations along with an estimate on the prediction error. We focus on two problems: sample placement and shortest tours in random fields. Both problems deal with the same constraint: ensure the estimation error at each point in the environment is within a desired tolerance. In sample placement, we seek to find a subset of measurement locations of minimum size that satisfies this constraint. In a similar vein, the shortest tour problem aims to compute the minimum length tour whose vertices satisfy the constraint. Unfortunately, sampling and tour planning in random fields are NP-hard \cite{das2008algorithms,singh2009efficient} and thus, our objective is to develop approximation algorithms.

\begin{figure}
    \centering
    \includegraphics[width=0.98\linewidth]{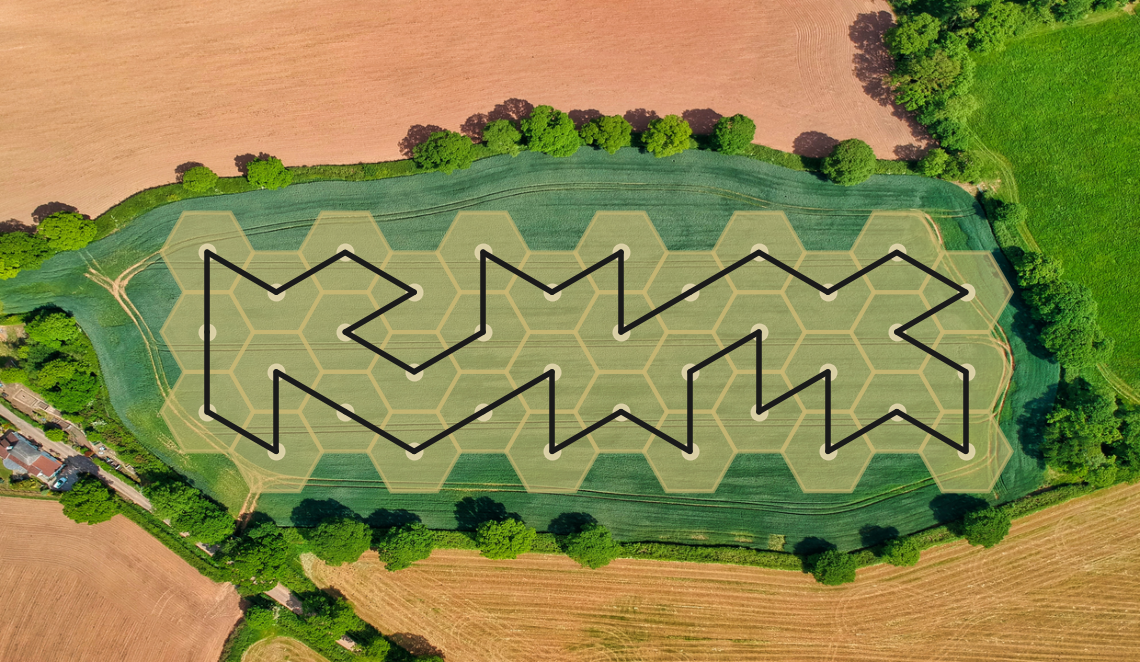}
    \caption{An example of a solution to the sample placement and shortest tour problem. The vertex set (white circles) of the tour (black) guarantees accurate predictions in the hexagons (yellow), thereby covering the field.}
    \label{fig:intro_fig}
\end{figure}
Our approach, \textsc{\hc}, to solving the sample placement problem involves tiling the environment with hexagons whose centers serve as measurement locations. The edge length of the hexagons are selected in a way such that the resulting measurement set satisfies the error tolerance. For the shortest tour problem, our proposed method, \textsc{\hct}, uses the measurement set obtained by \textsc{\hc} to plan an approximately optimal tour. An example of a sample placement and tour are shown in Figure~\ref{fig:intro_fig}.

\emph{Related Work:}
In geostatistics, estimation of random fields is done via the theory of kriging interpolation \cite{cressie2015statistics,stein1999interpolation}. With knowledge of the mean and covariance, estimation at unobserved locations is carried out using linear least-squares estimation. Since the expected error is independent of the observations, objectives based on the error can be computed a priori. This is equivalent to Gaussian Process regression in machine learning \cite{williams2006gaussian}. Given a set of measurements, we use kriging theory in this paper to make predictions using linear estimators and quantify the error at unobserved locations.

The sample placement and shortest tour problem have been considered previously \cite{suryan2020learning}. The analysis of the approximation algorithms relies on a claim of a lower bound on any solution to the sample placement problem. We provide a counterexample to disprove this claim. However, by using the effective range of the covariance function \cite{webster2007geostatistics,krause2008near,cao2013multi}, the provided approximation ratios still hold.

In sensor/sample placement, the objective is to distribute a set of sensors to maximize the sensing quality over a random field. Sensing quality is measured by the entropy and mutual information. Though computing the optimal subset in this case is NP-hard, greedy algorithms perform well \cite{ko1995exact} with guarantees stemming from submodularity of the objective \cite{krause2008near}. However, guarantees for mutual information do not hold for the estimation error which is the constraint in our formulation. While algorithms have been developed for sensor coverage \cite{feng2021sensor} in deterministic environments, the methods are not applicable in random fields.

In informative path planning, robots aim to maximize the information gained about the environment subject to constraints on path length. One approach to tackle continuous environments is by discretization where optimal solutions have been computed via branch-and-bound \cite{binney2012branch} and mixed-integer programming \cite{dutta2022informative}. When the objective is submodular, recursive greedy approaches \cite{chekuri2005recursive,binney2010informative,binney2013optimizing} provide near-optimal solutions. However, the estimation error is not a supermodular set function \cite{das2008algorithms,dutta2022}. A related kriging variance minimization problem is considered in \cite{xiao2022nonmyopic} but the approach does not provide approximation guarantees. Methods that operate in continuous space have also been developed \cite{hollinger2013sampling,meera2019obstacle} with guarantees on classification error \cite{tokekar2016sensor}. Certain scenarios require adaptive approaches for informative planning to adjust for outliers \cite{chen2022informative} or to the non-uniform spatial correlations \cite{zhu2021online,popovic2017multiresolution} or account for robot pose uncertainty \cite{popovic2020informative}.  Unfortunately, none of these algorithms provide guarantees for the estimation error constraint considered in this paper.

\emph{Contributions:}
The contributions of this work are twofold. First, we disprove a claim in the literature \cite{suryan2020learning} on a necessary condition for the sample placement problem. Second, we give approximation algorithms for the sample placement and shortest tour problems in convex environments, improving previous theoretical guarantees. In simulations, we demonstrate the effectiveness of our approach over prior work.

\section{Preliminaries} \label{section:preliminaries}
Consider a metric space $(\mathcal{X}, \rho)$ where $\rho: \mathcal{X} \times \mathcal{X} \rightarrow \mathbb{R}_{\geq 0}$ is a metric. Let $A \subset \mathcal{X}$ be a compact set and let $B(x, r) := \{ y \in \mathcal{X}: \rho(x, y) \leq r\}$ denote a closed ball. 

\subsection{Covering \& Packing Number} \label{subsec:covering_packing}

\begin{definition}[$r$-packing]
The set $\{x_1, \ldots, x_n\} \subset A$ is an $r$-packing of $A$ if the sets $\{ {B}(x_i, r/2): i \in [n] \} $ are pairwise disjoint i.e., for any $i \neq j$, $\rho(x_i, x_j) > r$. 
\end{definition}

\begin{definition}[$r$-covering]
The set $\{x_1, \ldots, x_n\} \subset A$ is an $r$-covering of $A$ if $A \subset \bigcup_{i=1}^n B(x_i, r)$ i.e., $\forall x \in A, \exists i$ such that $\rho(x_i, x) \leq r$. 
\end{definition}

\subsection{The Metric Traveling Salesman Problem} \label{subsec:metric_tsp}
The input to the metric traveling salesman problem (TSP) is a complete graph $G = (V, E)$, where the vertex set $V$ lies in a metric space $(\mathcal{X}, \rho)$. A cycle is a sequence of distinct vertices $\langle v_1, \ldots, v_n, v_1 \rangle$ and a tour $T$ is a cycle that visits each vertex in the graph exactly once. The cost of a tour, denoted by $\text{len}(T)$, is the sum of the costs of its associated edge set i.e., $\text{len}(T) := \sum_{i=1}^{n-1} \rho(v_i, v_{i+1}) + \rho(v_n, v_1)$. The objective of the TSP is then to find a tour of minimum cost. With slight abuse of notation, we let $T$ refer to the tour as well the set of vertices visited on the tour. 
\section{Problem Formulation} \label{section:problem_formulation}

We follow a similar problem setup to \cite{dutta2022,dutta2022informative,suryan2020learning}. We provide a summarized description here.
Consider a zero-mean Gaussian random field $\{ Z(x): x \in D\}$ indexed on a convex and compact environment $D \subset \mathbb{R}^2$ with a non-empty interior. Let $\rho: \mathbb{R}^2 \times \mathbb{R}^2 \rightarrow \mathbb{R}_{\geq 0}$ be the Euclidean metric and denote the squared exponential covariance function associated with the field by $\phi: \mathbb{R} \rightarrow \mathbb{R}_{\geq 0}$ i.e., for any $x, y \in D$,
\begin{equation}
    \text{Cov}\left( Z(x), Z(y) \right) = \phi(\rho(x, y)) = \sigma_0^2 e^{-\frac{\rho^2(x,y)}{2L^2}},
\end{equation}
where $\sigma_0 > 0, L > 0$ are known hyperparameters, typically determined by a pilot study.

It is common to assume that points that are sufficiently distant are uncorrelated since the covariance decays exponentially with the distance \cite{krause2008near,cao2013multi}. In spatial statistics, this is known as the \textit{effective range} of the covariance function \cite{webster2007geostatistics} and is usually taken to be the distance at which $\phi(\cdot)$ equals $5\%$ of the a priori field variance $\sigma_0^2$.
\begin{assumption}[Effective Range, \cite{webster2007geostatistics}] \label{assumption:rmax}
The distance $r_{\max}$ beyond which any two points are assumed to be uncorrelated is given by
\begin{equation}
    \phi(r_{\max}) =  0.05 \times \sigma_0^2 \implies r_{\max} = \sqrt{6} L.
\end{equation}
\end{assumption}

The measurements of the random field are corrupted with zero-mean Gaussian noise with variance $\sigma^2 > 0$. Given a measurement set $S = \{x_1, \ldots, x_n\} \subset D$, the optimal linear least-squares estimate of $Z(x)$ is a linear combination of the observations taken at $S$. The estimation error $f_x(S)$ is:
\begin{equation} \label{equation:estimation_error}
    \begin{split}
         f_x(S) &:= \phi(0) - \boldsymbol{b}_{x,S}' C_S^{-1} \boldsymbol{b}_{x,S},
    \end{split}
\end{equation}
where the following notation is used: 
\begin{equation}
    \begin{split}
        \boldsymbol{b}_{x,S} &:= [\phi(\rho(x,x_1)), \ldots, \phi(\rho(x,x_n))] \in \mathbb{R}^n\\
        C_S &:= \mathbb{E} \left[ \boldsymbol{Z}_S \boldsymbol{Z}_S^T\right] + \sigma^2 I_n \in \mathbb{R}^{n \times n}\\
        &= \begin{bmatrix}
        \phi(0) & \dots & \phi(\rho(x, x_1))\\
        \vdots & \ddots & \vdots\\
        \phi(\rho(x_n,x_1) ) & \dots & \phi(0)
        \end{bmatrix} + \sigma^2 I_n\\
    \end{split}
\end{equation}

We are ready to state the two problems we wish to solve.
\begin{problem}[Sample Placement] \label{prob:minimum_measurement_set}
Given an environment $D \subset \mathbb{R}^2$ and $\Delta < \sigma_0^2$, find a measurement set $S \subset D$ of minimum cardinality such that the resulting estimation error for each $x \in D$ is within $\Delta$ i.e.,
\begin{equation} \label{problem_formulation}
    \begin{aligned}
        & \underset{S \subset D}{\text{minimize}}
        & & |S|\\
        & \text{subject to}
        & & f_x(S) \leq \Delta, \; \forall x \in D.
    \end{aligned}
\end{equation}
\end{problem}

\begin{problem}[Shortest Tour] \label{prob:minimum_length_tour}
Given an environment $D \subset \mathbb{R}^2$ and $\Delta < \sigma_0^2$, find a tour $T$ of minimum length such that the vertex set of the tour achieves an estimation error within $\Delta$ for each point $x \in D$ i.e.,
\begin{equation} \label{problem_formulation_tours}
    \begin{aligned}
        & \underset{T}{\text{minimize}}
        & & \text{len}({T})\\
        & \text{subject to}
        & & f_x(T) \leq \Delta, \; \forall x \in D.
    \end{aligned}
\end{equation}
\end{problem}

Throughout this paper, we will denote the optimal solutions to the sample placement and shortest tour problem  
by $S^*$ and $T^*$ respectively.

\section{Solution Approach} \label{section:approach}

We begin with a discussion of a previous approach \cite{suryan2020learning} to solving the sample placement problem. In Section \ref{subsec:correction}, we disprove a claimed lower bound on the optimal solution. In light of this, we also observed the proposed algorithm in \cite{suryan2020learning} is sample inefficient. This motivates the development of new algorithms: \textsc{\hc} for the sample placement problem (Section \ref{subsec:alg_prob1}) and \textsc{\hct} for the shortest tour problem (Section \ref{subsec:alg_prob2}).

\subsection{Disproving a Necessary Condition} \label{subsec:correction}

The approach for sample placement proposed in~\cite{suryan2020learning} is called \textsc{\dc} and uses two steps. First, it covers the environment with large circles of a certain radius $r_{\max}$. Second, it then covers these circles with smaller circles to obtain a feasible solution which yields an approximation algorithm. The analysis does not use Assumption \ref{assumption:rmax}. Instead, it involves a claim on a necessary condition for the problem. We provide a counterexample to disprove this claim.

It would be useful to obtain a distance $r_{\max}$ such that if a prediction location does not have a sample within $r_{\max}$, then the constraint in \eqref{problem_formulation} cannot be satisfied. Then, a packing of circles of radius $r_{\max}$ provides a lower bound for the sample placement problem. This is the idea behind the necessary condition in \cite{suryan2020learning}. We state the lemma below.

\begin{lem}[Lemma 1, \cite{suryan2020learning}]
For any test location $x$, if the nearest measurement location is at a distance $r_{\max}$ away, and
\begin{equation} \label{ce:lemma1}
    r_{\max} > L \sqrt{- \log \left(1 - \frac{\Delta}{\sigma_0^2}\right)},
\end{equation}
then it is not possible to bring down the MSE (estimation error) below $\Delta$ at $x$.
\end{lem}

The proof constructs a bound for the estimation error by placing all samples at the nearest location to the test location $x$. This is shown in the left plot in Figure~\ref{fig:counterexample}. While this seems intuitive, this does not give a valid lower bound on the error. We construct a measurement set, as shown in the right plot in Figure~\ref{fig:counterexample}, that obtains a lower estimation error.
\begin{exmp}
Consider the following environment setup: $\sigma = 1, \sigma_0 = 1, L = 1, x = (0,0), \Delta = 0.5$. Then, the RHS of (\ref{ce:lemma1}) is $0.83255461$. Let $r_{\max} = 0.93255461$ and consider the following measurement set:
\begin{equation}
    S := \left\{(r_{\max}, 0), (-r_{\max}, 0), (0, r_{\max}), (0, -r_{\max}) \right\}.
\end{equation}
The estimation error at $x$ using measurement set $S$ is
\begin{equation}
    \begin{split}
    f_x(S) &= 0.443771 < \Delta,
    \end{split}
\end{equation}
which shows the contradiction.
\oprocend
\end{exmp}

\begin{figure}
    \centering
    \includegraphics[width=0.95\linewidth]{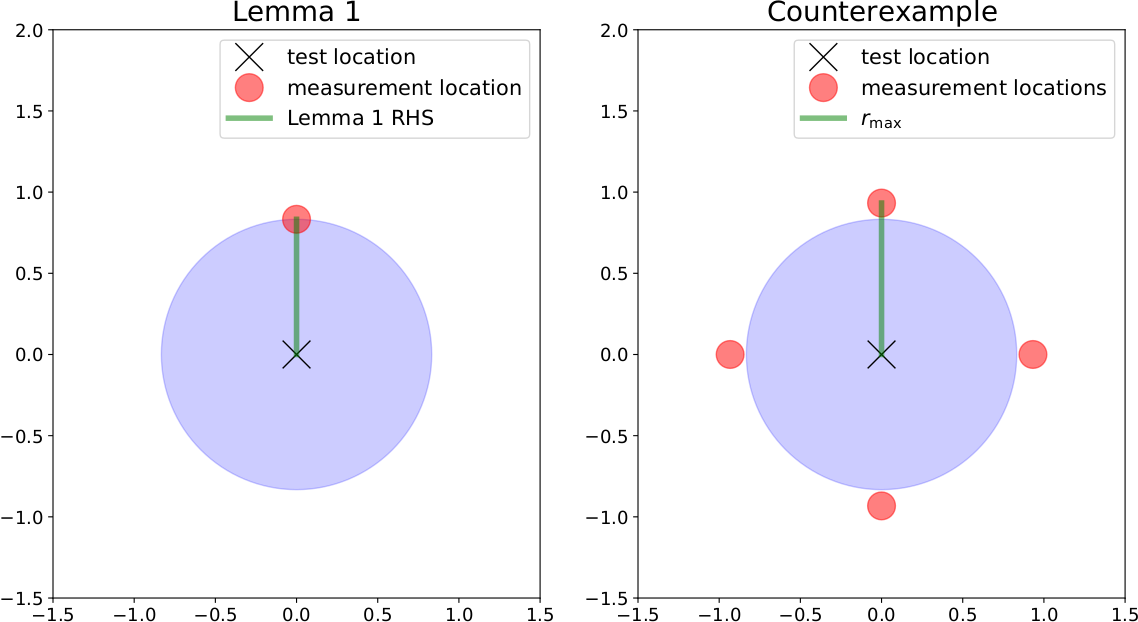}
    \caption{Visual depiction of the counterexample. Left: Lemma 1 in \cite{suryan2020learning}, constructs a lower bound by placing samples at the red circle. Right: The set of four measurement locations obtains a lower estimation error.}
    \label{fig:counterexample}
\end{figure}

The counterexample uses four points to invalidate the lower bound on $r_{\max}$. However, by placing a large number of samples outside $r_{\max}$, the error can be much lower, thereby underestimating the lower bound on $r_{\max}$ by a large margin. By using Assumption \ref{assumption:rmax}, which is a common practical assumption in geostatistics \cite{webster2007geostatistics}, the approximation ratios in \cite{suryan2020learning} still hold (though Lemmas 1 and 3 in \cite{suryan2020learning} do not hold). Since the proposed algorithm \cite{suryan2020learning} covers the environments with circles of radius $r_{\max}$ (which can be large), and covers it with smaller circles, it uses an excessive number of samples. This motivates us to develop an alternative algorithm.

\subsection{\textsc{\hc}} \label{subsec:alg_prob1}
Our approach to solving the sample placement problem is simple: cover the environment with circles of a specific radius $r_{\min}$. Then, one solution is the set of circle centers. We shall see that an appropriate selection of $r_{\min}$ guarantees feasibility of this approach. 

\begin{lem}[Sufficient Condition]
Let $S \subset D$ be a $r_{\min}$-covering of the environment $D$, where
\begin{equation} \label{eq:rmin}
    r_{\min} := L \sqrt{ - \log \Big( (\sigma_0^2 - \Delta) \frac{\sigma_0^2 + \sigma^2}{\sigma_0^4} \Big)},
\end{equation}
then $S$ is feasible for the sample placement problem.
\end{lem}
\begin{proof}
We will show that for any point $x \in D$ in the environment, the estimation error $f_x(S) \leq \Delta$. Since $S$ is a $r_{\min}$-covering, there exists a point $x_i \in S$ such that $\rho(x, x_i) \leq r_{\min}$. In addition, since $f_x(S)$ is a monotonically decreasing set function \cite{krause2008robust}, it suffices to show $f_x(\{x_i\}) \leq \Delta$. Then,
    \begin{equation*}
        \begin{split}
            \rho(x , x_i) \leq r_{\min}
            \implies e^{-\frac{\rho(x, x_i)^2}{L^2}} &\geq (\sigma_0^2 - \Delta) \frac{\sigma_0^2 + \sigma^2}{\sigma_0^4}\\ \implies f_x(\{x_i\}) &\leq \Delta.
        \end{split}
    \end{equation*}
\end{proof}

The last step is to generate the $r_{\min}$-covering. Since the objective is to minimize the number of samples, one idea is to compute the minimum $r_{\min}$-covering. Unfortunately, this is difficult to compute in general. Instead, we compute a cover from a hexagonal tiling, which is the densest way to pack circles in the Euclidean plane \cite{fejes1942dichteste,chang2010simple}. The steps are described in Algorithm~\ref{alg:hex_cover}. The main step is $\textsc{HexagonalTiling}$ (Line 2), which returns the set of centers in the hexagonal tiling of edge length $r_{\min}$. The centers are arranged in staggered columns where the vertical distance between two samples is $\sqrt{3} r_{\min}$ and the horizontal spacing between the columns is $1.5 r_{\min}$. An example of a hexagonal tiling is shown in Figure~\ref{fig:intro_fig}. By construction, the circles centered at these points will circumscribe the hexagon, thereby covering the environment. Note that $\textsc{\hc}$ runs in polynomial time since it scales linearly with the area of the environment i.e., the time complexity is $O\left(\text{area}(D)\right)$.

\begin{algorithm}	
	\DontPrintSemicolon 
	\KwIn{Environment $D \subset \mathbb{R}^2$, Error Threshold: $\Delta$, Hyperparameters: $L, \sigma_0^2$, Noise Variance $\sigma^2$}
	\KwOut{Measurement Set: $S \subset D$}		
	Compute $r_{\min}$ according to Equation \eqref{eq:rmin}.\\
	$S_{\text{HC}} = \textsc{HexagonalTiling}(D, r_{\min})$\\
	\Return{$S_{\text{HC}}$}
	\caption{\textsc{\hc}}
	\label{alg:hex_cover}
\end{algorithm}

\begin{assumption}[Boundary Conditions] \label{rem:boundary_condition}
For some environments, there exist regions close to the boundary that are not covered by the hexagonal tiling. This is solved by taking a few more measurements to ensure feasibility. However, for the purposes of algorithm analysis, we will assume that the hexagonal tiling (Line 2) covers the environment completely.
\end{assumption}

To analyse the performance of this approach, the first step is to compute an upper bound on the solution $|S_{\text{HC}}|$ returned by \textsc{\hc}. To do this, we require the definition of a Minkowksi sum and a lemma bounding the maximum area of the Minkowski sum of a convex set and a unit circle. Our proof of this result is based on the lecture notes~\cite{lecturenotes} and is included here for completeness.

Let $B := \{(x,y): x^2 + y^2 \leq 1\}$ be the unit circle and for any convex set $\Theta$ and $c > 0$, define $c\Theta := \{cx : x \in \Theta\}$.

\begin{definition}[Minkowski Sum]
For any two sets $A, B$, the Minkowski sum is $A + B := \{x + y: x \in A, y \in B\}$.
\end{definition}
\begin{lem} \label{lem:upper_bound_area_convex}
Given a convex set $\Theta \subset \mathbb{R}^2$, the unit circle $B$, and $r > 0$ with $r B \subset \Theta$, then \[ \text{area}\left(\Theta + \frac{r}{2} B \right) \leq \frac{9}{4} \text{area}\left(\Theta\right). \]
\end{lem}
\begin{proof}
Take any $a \in \Theta + \frac{r}{2} B$. Then, it can be expressed as $a = x + y$, where $x \in \Theta, y \in \frac{r}{2}B$. Since $r B \subset \Theta$, then $\frac{r}{2} B \subset \frac{1}{2} \Theta$. Thus, $a \in \Theta + \frac{1}{2}\Theta$. Now, $a$ can be expressed as $a = s + \frac{t}{2}$, where $s, t \in \Theta$. Then, $\frac{2}{3} a = \frac{2}{3} s + \frac{t}{3}$. Since $\Theta$ is convex, we have $\frac{2}{3} a \in \Theta$ which implies $a \in \frac{3}{2} \Theta$. Thus, $\Theta + \frac{r}{2} B \subset \frac{3}{2} \Theta$ from which the final result follows.
\end{proof}

Based on Lemma \ref{lem:upper_bound_area_convex} we now establish an upper bound on the number of measurements $|S_\text{HC}|$.
\begin{lem}[Upper Bound] \label{lem:upperbound_hexcover}
Let $\delta > 0$ be arbitrarily small and let $(\sqrt{3}-\delta) r_{\min} B \subset D$, where $B$ is the unit circle. Then for some arbitrarily small $\epsilon_\delta > 0$,
\begin{equation*}
    |S_\text{HC}| \leq (3+\epsilon_\delta)\frac{\text{area}(D)}{\pi r_{\min}^2}.
\end{equation*}
\end{lem}
\begin{proof}
Denote the solution $S_{\text{HC}} := \{x_1, \ldots, x_k\}$. Then, under Assumption~\ref{rem:boundary_condition}, $S_\text{HC}$ is a $(\sqrt{3}-\delta) r_{\min}$-packing for the environment $D$. This is because the minimum distance between any two points in the hexagonal cover is $\sqrt{3} r_{\min}$ i.e., $ {\min_{i \neq j}}\ \rho(x_i, x_j) = \sqrt{3} r_{\min} > (\sqrt{3}-\delta) r_{\min}$,  which ensures it is a $(\sqrt{3}-\delta) r_{\min}$-packing. Since the circles in the packing are disjoint, we have $\bigcup_{i}^k B(x_i, \frac{(\sqrt{3}-\delta)}{2} r_{\min}) \subset D + \frac{(\sqrt{3}-\delta)}{2} r_{\min} B$. Then, letting $c = \sqrt{3} - \delta$ and taking the area on both sides gives
\begin{equation}
    \begin{split}
        \text{area}\left(\bigcup_{i}^{|S_{\text{HC}}|} B(x_i, cr_{\min}/2)\right) &= |S_{\text{HC}}| \frac{\pi c^2 r^2_{\min}}{4}\\
        \leq \text{area}\left(D + \frac{cr_{\min}}{2} B\right) &\leq  \frac{9}{4} \text{area}(D)\\
        \implies |S_{\text{HC}}| &\leq \left(\frac{3}{c}\right)^2 \frac{\text{area}(D)}{\pi r^2_{\min}}\\
                        &= (3+\epsilon_\delta) \frac{\text{area}(D)}{\pi r^2_{\min}}, 
    \end{split}
\end{equation}
where the second inequality follows from Lemma~\ref{lem:upper_bound_area_convex} and \[\epsilon_\delta =  \left(\frac{3}{\sqrt{3}-\delta}\right)^2 - 3\] is arbitrarily small.
\end{proof}

To obtain an approximation factor for Algorithm \ref{alg:hex_cover}, we characterize an optimal solution $S^*$. First, we establish a property on any feasible solution which we will then use to get a lower bound on $|S^*|$.
\begin{lem}[Feasible Sample Placement] \label{lem:feasible_sample}
Any feasible solution $S$ is a $r_{\max}$-covering of the environment $D$.
\end{lem}
\begin{proof}
Suppose not. Then, there exists $x \in D$ such that for all $y \in S$, $\rho(x,y) > r_{\max}$. Using Assumption \ref{assumption:rmax}, the estimation error at $x$ using $S$ is $f_x(S) = \sigma_0^2 > \Delta$, contradicting feasibility of $S$ for the sample placement problem.
\end{proof}
\begin{lem}[Lower Bound] \label{lem:lower_bound}
A lower bound on the optimal value $|S^*|$ for the sample placement problem is
\begin{equation*}
    |S^*| \geq \frac{\text{area}(D)}{\pi r^2_{\max}}.
\end{equation*}
\end{lem}
\begin{proof}
Using Lemma \ref{lem:feasible_sample}, $S^*$ is a $r_{\max}$-covering of the environment $D$ i.e., $\bigcup_{x \in S^*} B(x, r_{\max}) \supset D$. Taking area on both sides gives
\begin{equation*}
    \begin{split}
 \text{area}(D)  &\leq     \text{area}\left(\bigcup_{x \in S^*} B(x, r_{\max})\right)\\
 &\leq \sum_{x \in S^*} \text{area}\left(B(x, r_{\max})\right)
        = |S^*| \pi r^2_{\max},
    \end{split}
\end{equation*}
where the second inequality follows from the fact that area is a sub-additive function.
\end{proof}

\begin{thm} \label{thm:hexcover}
For an arbitrarily small $\epsilon>0$, \textsc{\hc} is an $\alpha$-approximation algorithm for the sample placement problem where 
    \begin{equation*}
        \alpha := 3\frac{r^2_{\max}}{r^2_{\min}} + \epsilon.
    \end{equation*}
\end{thm}
\begin{proof}
Using Lemmas \ref{lem:upperbound_hexcover} and \ref{lem:lower_bound}, for some arbitrarily small $\epsilon_\delta > 0 $, $|S_{\text{HC}}| \leq (3+\epsilon_\delta) \frac{\text{area}(D)}{\pi r^2_{\min}} \leq (3+\epsilon_\delta) \frac{r^2_{\max}}{r^2_{\min}} |S^*| = (3 \frac{r^2_{\max}}{r^2_{\min}} + \epsilon)|S^*|$, where $\epsilon = \epsilon_\delta \frac{r^2_{\max}}{r^2_{\min}}$ is arbitrarily small.
\end{proof}
Theorem \ref{thm:hexcover} improves the previous approximation ratio of $18 \frac{r_{\max}^2}{r^2_{\min}}$ \cite{suryan2020learning}. Our approximation factor is independent of the environment $D$ but does depend on the error threshold $\Delta$ through $r_{\min}$. As we seek more accurate predictions (decreasing $\Delta)$, the radius of accurate estimation $r_{\min}$ reduces, and the number of measurements required will increase.

\subsection{\textsc{\hct}} \label{subsec:alg_prob2}
Our approach to solving the shortest tour problem builds on \textsc{\hc}. This is because the constraints are the same in both problems. Since we have already identified a feasible vertex set via \textsc{\hc}, our method uses Christofides' algorithm which finds a tour of length no more than 3/2 times the optimal \cite{korte2011combinatorial}. The steps are outlined in Algorithm~\ref{alg:hex_cover_tour}. The time complexity is dominated by Christofides' algorithm (Line 2) which runs in $O(n^3)$ time, where $n$ is the number of vertices in the graph. Since the number of vertices is $O(\text{area}(D))$, \textsc{\hct} runs in $O(\left(\text{area}(D)\right)^3)$ time.
\begin{algorithm}	
	\DontPrintSemicolon 
	\KwIn{Environment $D \subset \mathbb{R}^2$, Error Threshold: $\Delta$, Hyperparameters: $L, \sigma_0^2$, Noise Variance $\sigma^2$}
	\KwOut{Tour $T$ with $T \subset D$.}
	$S_{\text{HC}} = \textsc{\hc}(D, \Delta, L, \sigma_0^2, \sigma^2)$\\
	$T_{\text{HC}} = \textsc{ChrisotofidesTour}(S_\text{HC})$\\
	\Return{$T_{\text{HC}}$}
	\caption{\textsc{\hct}}
	\label{alg:hex_cover_tour}
\end{algorithm}

We begin by characterizing an upper bound on the length of the tour $T_{\text{HC}}$ produced by \textsc{\hct}.
\begin{lem}[Upper Bound] \label{lem:hexcovertour_upper_bound}
For an arbitrarily small $\epsilon > 0$, \[ \text{len}(T_{\text{HC}}) \leq {15.6} \frac{\text{area}(D)}{\pi r_{\min}} + \epsilon. \]
\end{lem}
\begin{proof}
Denote the vertex set of the tour by $V_\text{HC} := \{ x_1, \ldots, x_k \}$ and let $T_k$ denote the optimal TSP tour on the $k$ vertices of the tour $T_\text{HC}$. Note that in the hexagonal covering, the minimum distance between two points is $\sqrt{3} r_{\min}$. Now, we can bound $T_k$ by constructing a sub-optimal tour as follows. Take any two points $x_i, x_j \in V_\text{HC}$ such that $\rho(x_i, x_j) \leq \sqrt{3} r_{\min}$ and find the shortest tour in $V_{\text{HC}} \setminus x_i$. Then, we can create a tour by adding the edge to connect $x_j$ to $x_i$ and back. This gives an upper bound: 
\begin{equation*}
    \begin{split}
        \text{len}(T_k) &\leq \text{len}(T_{k-1}) + 2 \sqrt{3} r_{\min}.
    \end{split}
\end{equation*}
Using the fact that $\text{len}(T_1) = 0$ and Lemma \ref{lem:upperbound_hexcover}, we get \[ \text{len}(T_k) \leq  2 \sqrt{3} r_{\min} |V_\text{HC}| \leq 2 \sqrt{3} (3+\epsilon_\delta) \frac{\text{area}(D)}{\pi r_{\min}}.\]

Since we are using Christofides' algorithm, we have 
\begin{equation*}
    \begin{split}
        \text{len}(T_\text{HC}) \leq 1.5 \text{len}(T_k) &\leq 3 \sqrt{3} (3+\epsilon_\delta) \frac{\text{area}(D)}{\pi r_{\min}}\\
        & = 15.6 \frac{\text{area}(D)}{\pi r_{\min}} + \epsilon,
    \end{split} 
\end{equation*}
where $\epsilon = 3 \sqrt{3} \frac{\text{area}(D)}{\pi r_{\min}} \epsilon_\delta$ is an arbitrarily small number. 
\end{proof}

In the following lemma, we compute a lower bound on the length of the optimal tour $T^*$. First, we need a property of a maximal packing of a $r_{\max}$-covering.
\begin{lem} \label{lem:maximal_cover}
Let $S := \{x_1, \ldots, x_n\}$ be a $r_{\max}$-covering of the environment $D$. Then, any maximal $2r_{\max}$-packing $P \subset S$ of $S$ is a $3 r_{\max}$-covering of $D$.
\end{lem}
\begin{proof}
Using the triangle inequality and the fact that $S$ is a $r_{\max}$-covering of $D$, it suffices to show that for any $x_i \in S$, there exists a point in the packing $y \in P$ such that $\rho(x_i, y) \leq 2 r_{\max}$. Suppose not. Then, there exists $x_i \in S$ such that for all $y \in P$, $\rho(x_i, y) > 2r_{\max}$. But, we could add $x_i$ to $P$ and increase its size, contradicting the fact it is maximal.
\end{proof}

\begin{lem}[Lower Bound] \label{lem:lower_bound_tour}
A lower bound on the length of the optimal tour is \[ \text{len}(T^*) \geq \frac{2}{9} \frac{\text{area}(D)}{\pi r_{\max}}.\]
\end{lem}
\begin{proof}
  We compute a maximal disjoint set $T' \subset T^*$ from the set of circles of radius $r_{\max}$ centered at points in the optimal tour $T^*$. By the triangle inequality, the optimal TSP tour through $T'$ gives us a bound on the length of the optimal tour: 
\begin{equation} \label{eq:b1}
    \text{len}(T^*) \geq \text{len}(T').
\end{equation}

Since $T'$ is a $2 r_{\max}$-packing, the minimum distance between any two vertices in $T'$ is $2r_{\max}$. Thus, 
\begin{equation} \label{eq:b2}
    \text{len}(T') \geq 2 r_{\max} |T'|.
\end{equation}

Since $T^*$ is a $r_{\max}$-covering of $D$ (Lemma \ref{lem:feasible_sample}), $T'$ is a $3r_{\max}$-covering (Lemma \ref{lem:maximal_cover}). Then, $\bigcup_{x \in T'} B(x, 3 r_{\max}) \supset D$ which implies 
\begin{equation} \label{eq:b3}
    \begin{split}
    \text{area}\left(\bigcup_{x \in T'} B(x, 3 r_{\max})\right) &\geq \text{area}(D)\\
    \implies \sum_{x \in T'} \text{area}\left(B(x, 3 r_{\max})\right) &\geq \text{area}(D)\\
    \implies |T'| &\geq \frac{\text{area}(D)}{9 \pi r^2_{\max}}.
    \end{split}
\end{equation}
Combining Equations \eqref{eq:b1}, \eqref{eq:b2}, and \eqref{eq:b3} gives us the result.
\end{proof}

The guarantee is a consequence of Lemmas \ref{lem:hexcovertour_upper_bound} and \ref{lem:lower_bound_tour}. 
\begin{thm}
\textsc{\hct} is an $\alpha$-approximation algorithm for the shortest tour problem where for an arbitrarily small $\epsilon > 0$, \[ \alpha :=  70.2 \frac{r_{\max}}{r_{\min}} + \epsilon.\]
\end{thm}
This improves the previous approximation ratio of $9.33 + O(\frac{r^2_{\max}}{r^2_{\min}})$ \cite{suryan2020learning}. Note that if we use good heuristic solvers, such as the Lin-Kernighan heuristic \cite{korte2011combinatorial}, we can reduce our factor by (roughly) 1.5 leading to a $(46.8 \frac{r_{\max}}{r_{\min}} + \epsilon)$-approximation.

\section{Simulations} \label{section:simulations}
In this section, we demonstrate the effectiveness of \textsc{\hc}, \textsc{\hct} over \textsc{\dc}, \textsc{\dct} \cite{suryan2020learning} across environments of different sizes.

\subsubsection*{Experimental Setup}
We follow the setup in \cite{suryan2020learning} where a Gaussian Process was fit to a real world dataset of organic matter measurements in an agricultural field. The authors computed $L = 8.33$ meters, $\sigma_0 = 12.87$, and $\sigma^2 = 0.0361$. Since the covariance function is stationary, only the relative distances between points matter. This enables us to consider different environment sizes similar to the setup in \cite{dutta2022}. We consider rectangular environments whose area ranges from 400 to 40000 square metres.  Further, we consider three regimes of desired accuracy: $\Delta/\sigma_0^2 = 0.3, 0.2, 0.1$. These correspond to keeping the estimation error under $30\%, 20\%,$ and $10\%$ of the initial value $\sigma_0^2$, respectively. The experiments are run on an AMD Ryzen 7 2700 processor. 
\begin{rem}
The \textsc{\dc} and \textsc{\dct} algorithms end up placing samples outside the environment, which is not practically possible. In the simulations considered, we remove these samples and the redundant measurement locations that arise (as mentioned in the paper \cite{suryan2020learning}). In addition, for \textsc{\dct}, we use approximation algorithms for the TSP instead of the proposed lawn-mower tours. This change leads to shorter tour lengths in practice.
\end{rem}

\subsection{Sample Placement}
\begin{figure}
    \centering
    \includegraphics[width=\linewidth]{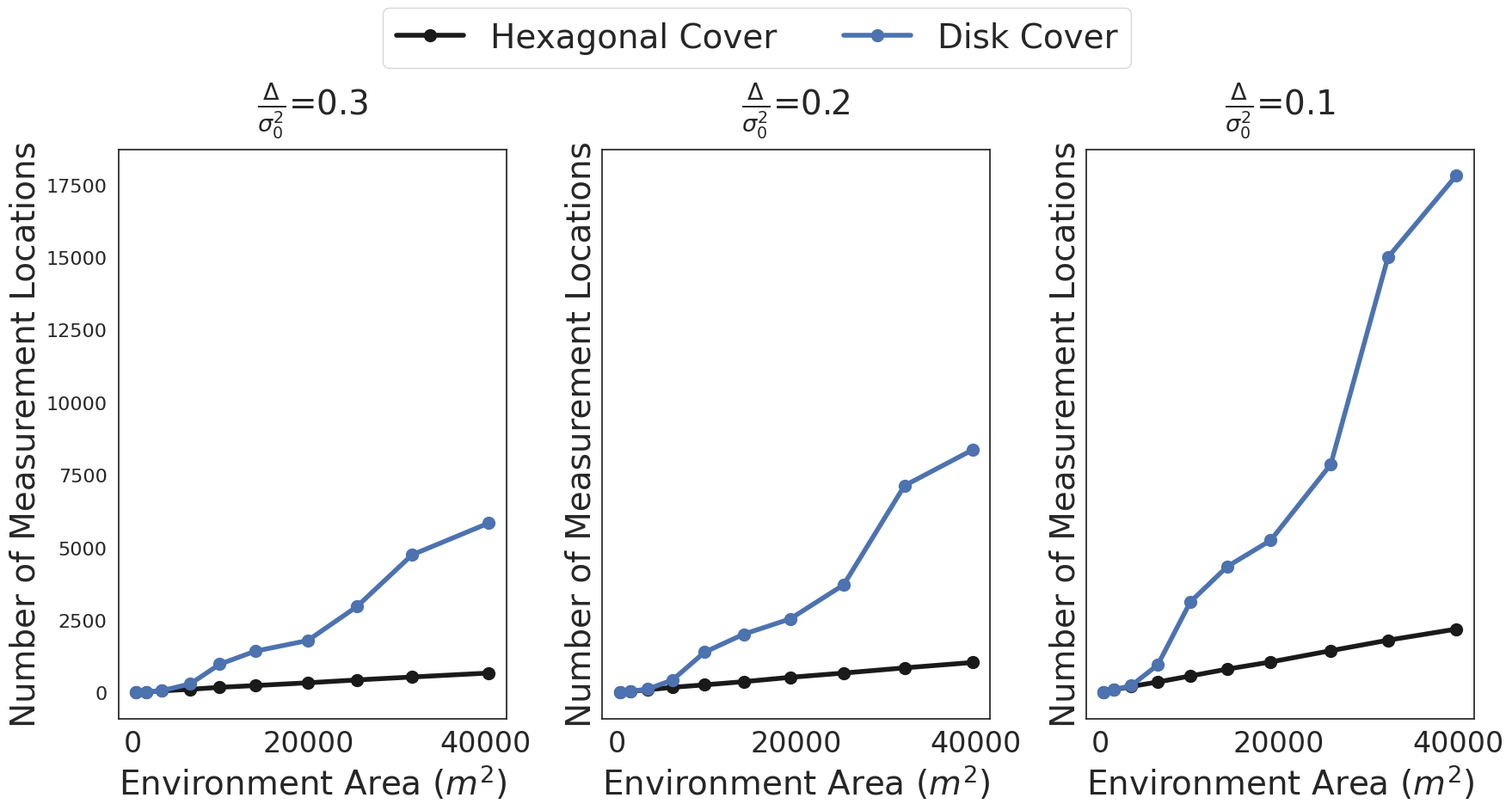}
    \caption{Comparison of the number of measurement locations for Problem \ref{prob:minimum_measurement_set} versus the environment size. \textsc{\dc} uses roughly 6 times as many measurements as \textsc{\hc}.}
    \label{fig:experiment1}
\end{figure}

The results for the sample placement problem are shown in Figure~\ref{fig:experiment1}. In each subplot, the number of measurements (y-axis) is plotted against the environment size (x-axis). The difference across the subplots is the error tolerance. As the error tolerance decreases, the number of measurements required will increase for any algorithm.
In each subplot, we observe that \textsc{\dc} uses more measurements than \textsc{\hc} (ours). In small environments, the difference is negligible as it does not take many measurements to get a feasible solution. However, as the environment size increases, \textsc{\dc} uses roughly 6 times as many measurements as \textsc{\hc}. Further, the number of measurements used by \textsc{\hc} increases roughly linearly with the environment size. This is not the case for \textsc{\dc} whose measurement set size grows much faster.

\subsection{Shortest Tours}

\begin{figure}
    \centering
    \includegraphics[width=\linewidth]{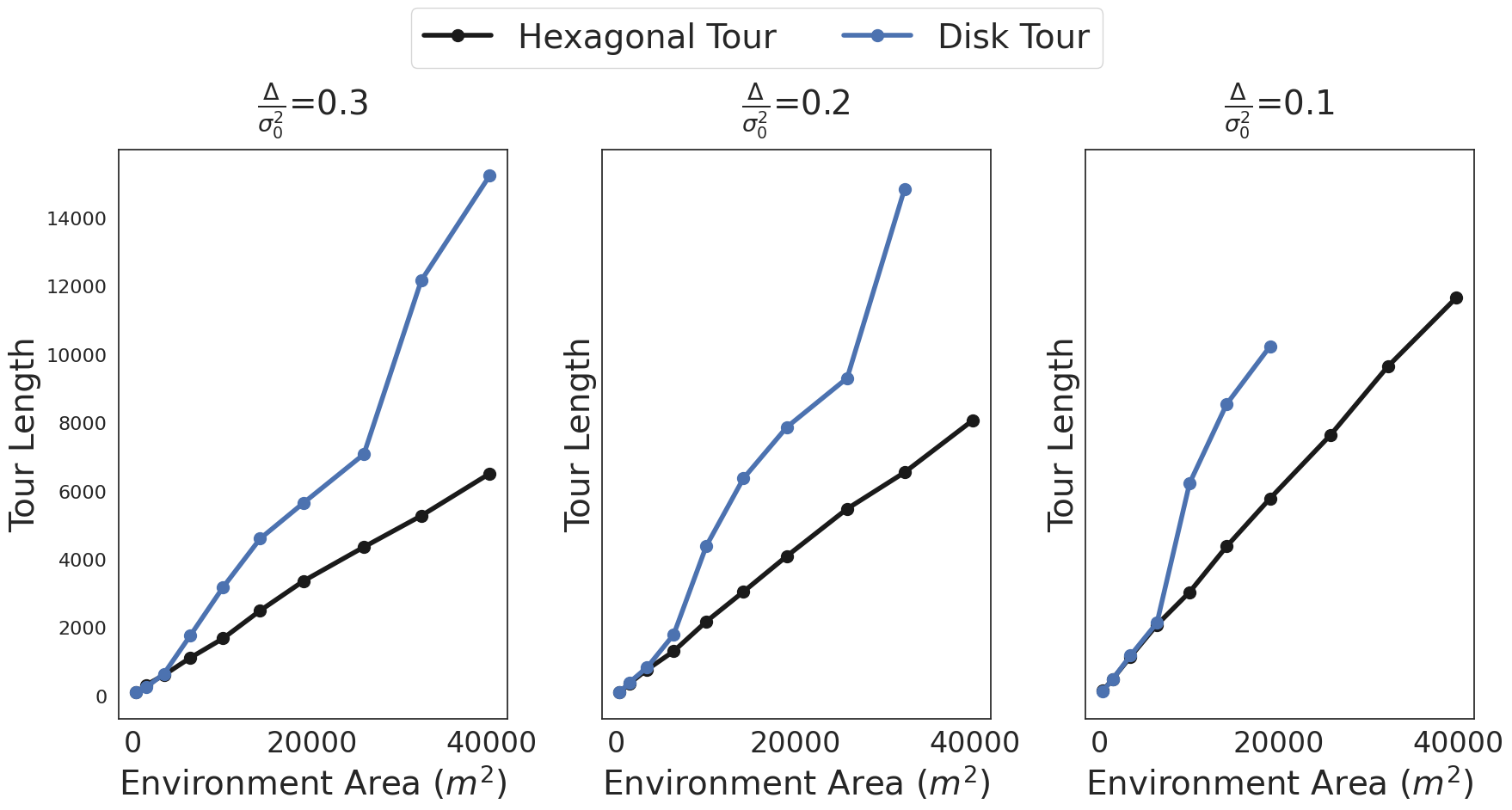}
    \caption{Comparison of the tour length vs environment size for Problem \ref{prob:minimum_length_tour}. \textsc{\hct} plans tours of lengths roughly half that of \textsc{\dct}. }
    \label{fig:experiment2}
\end{figure}

The results for the shortest tour problem are shown in Figure~\ref{fig:experiment2}. In each subplot, the tour length (y-axis) is plotted against the environment size (x-axis). As before, the difference across the subplots is the error tolerance. As the error tolerance decreases, the number of measurements required will increase, which subsequently increases the tour length.

In each subplot, we see that \textsc{\dct} produces tours of length larger than \textsc{\hct} (ours). For small environments, the difference in tour length is negligible. However, for larger environments, \textsc{\hct} produces tours that are roughly half the length of tours produced by \textsc{\dct}. Further, in the second and third plots, one may notice missing data points for \textsc{\dct}. This is because the number of vertices become too large (>8000) to run Christofides' algorithm in a reasonable amount of time and memory. In contrast, \textsc{\hct} scales better with the environment size.
\section{Conclusions \& Future Directions} \label{section:conclusions}

We considered the problem of finding the subset of measurement locations and the shortest tour in a random field that guarantees a desired level of prediction accuracy. We provided approximation algorithms for both problems based on hexagonal tiling of the Euclidean plane. In simulations, our algorithm outperforms prior work across environments of varying size. We also provided a counterexample to disprove a claim on a lower bound for the sample placement problem.
Looking forward, we are interested in two extensions. The first involves an importance weight associated with every field location leading to a weighted estimation error constraint requiring a new approach. Second, we are interested in the multi-robot problem for field estimation where we are given a team of robots tasked with accurate field mapping.

\Urlmuskip=0mu plus 1mu
\bibliographystyle{IEEEtran}
\bibliography{mybib}
\end{document}